\documentclass{article}

\usepackage[final]{neurips_2019}
\usepackage[utf8]{inputenc} % allow utf-8 input
\usepackage[T1]{fontenc}    % use 8-bit T1 fonts
\usepackage{hyperref}       % hyperlinks
\usepackage{url}            % simple URL typesetting
\usepackage{booktabs}       % professional-quality tables
\usepackage{amsfonts}       % blackboard math symbols
\usepackage{nicefrac}       % compact symbols for 1/2, etc.
\usepackage{microtype}      % microtypography

\usepackage{graphicx}
\usepackage{xcolor}
\usepackage{enumitem}
\usepackage{amsmath}
\usepackage{amssymb}
\usepackage{graphicx}
\usepackage[utf8]{inputenc}
\usepackage{bm}
\usepackage[ruled,vlined]{algorithm2e}
\usepackage{wrapfig}
\usepackage{float}
\usepackage[justification=centering]{caption}
\usepackage{listings}
\usepackage{subcaption}

\usepackage{amsthm}
\usepackage{thmtools}
\usepackage{thm-restate}
\usepackage{mathtools}

\declaretheorem[name=Definition]{definition}

\newcommand \TODO [1] {}

\newcommand{\E}{\mathbb E}
\newcommand{\KL}{D_\text{KL}}

\newcommand \mc \mathcal

\newcommand {\augm} [1] { \pmb{{#1}}}  % ^{\bm :}

\newcommand {\x} {\augm x}
\newcommand {\X} {\augm X}

\newcommand{\itup}[1]{{\setlength{\fboxsep}{-1.8pt}\colorbox{gray!17}{\strut $\displaystyle \ #1 \ $}}}

\newcommand{\dom}[1]{{\mathrlap{ #1 } \ }}

\newcommand {\hide} [1] {}
\title{Real-Time Reinforcement Learning}

\author{%
 Simon Ramstedt \\
  Mila, Element AI, \\
  Université de Montréal \\
  \texttt{simonramstedt@gmail.com} \\
   \And
  Christopher Pal \\
  Mila, Element AI, \\
  Polytechnique Montréal \\
  \texttt{christopher.pal@polymtl.ca} \\
}

\begin{document}

\maketitle

\begin{abstract}
Markov Decision Processes (MDPs), the mathematical framework underlying most algorithms in Reinforcement Learning (RL), are often used in a way that wrongfully assumes that the state of an agent's environment does not change during action selection. As RL systems based on MDPs begin to find application in real-world, safety-critical situations, this mismatch between the assumptions underlying classical MDPs and the reality of real-time computation may lead to undesirable outcomes. 
In this paper, we introduce a new framework, in which states and actions evolve simultaneously and show how it is related to the classical MDP formulation. We analyze existing algorithms under the new real-time formulation and show why they are suboptimal when used in real time. We then use those insights to create a new algorithm Real-Time Actor-Critic (RTAC) that outperforms the existing state-of-the-art continuous control algorithm Soft Actor-Critic both in real-time and non-real-time settings.
Code and videos can be found at \href{https://github.com/rmst/rtrl}{\texttt{github.com/rmst/rtrl}}.
\end{abstract}

Reinforcement Learning, has led to great successes in games \citep{tesauro1994td, mnih2015human, silver2017mastering} and is starting to be applied successfully to real-world robotic control \citep{schulman2015trust, hwangbo2019learning}.

\begin{wrapfigure}{r}{0.26\textwidth}

  \begin{center}
    \includegraphics[width=0.25 \textwidth, trim=0 30 0 80]{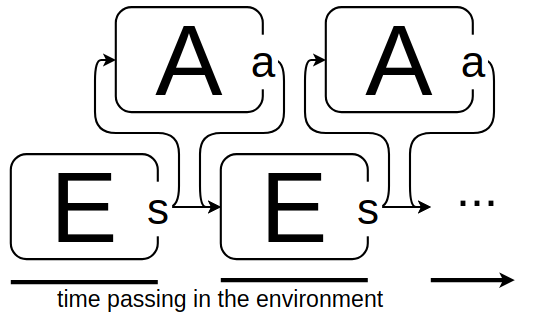}  % trim=left bottom right top
  \end{center}
 \caption{Turn-based interaction}
 \label{fig:TBDP}
\vspace{0.6cm}

\begin{center}
    \includegraphics[width=0.23 \textwidth, trim=0 50 0 0]{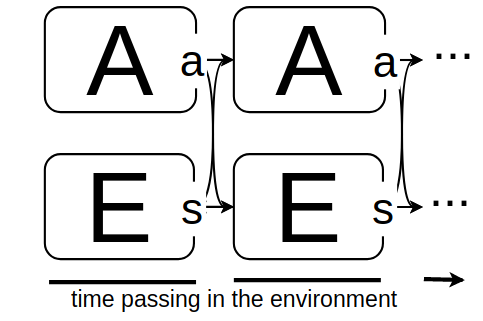}
 \end{center}
 \caption{Real-time interaction}
 \label{fig:RTDP}
 \vspace{-0.0cm}
% \end{figure}  
\end{wrapfigure}

The theoretical underpinning for most methods in Reinforcement Learning is the Markov Decision Process (MDP) framework \citep{bellman1957markovian}.
While it is well suited to describe turn-based decision problems such as board games, this framework is ill suited for real-time applications in which the environment's state continues to evolve while the agent selects an action \citep{travnik}. Nevertheless, this framework has been used for real-time problems using what are essentially tricks, e.g. pausing a simulated environment during action selection or ensuring that the time required for action selection is negligible \citep{hwangbo2017control}.
    
Instead of relying on such tricks, we propose an augmented decision-making framework - Real-Time Reinforcement Learning (RTRL) - in which the agent is allowed exactly one timestep to select an action. RTRL is conceptually simple and opens up new algorithmic possibilities because of its special structure.

We leverage RTRL to create Real-Time Actor-Critic (RTAC), a new actor-critic algorithm, better suited for real-time interaction, that is based on Soft Actor-Critic \citep{haarnoja2018soft}.
We then show experimentally that RTAC outperforms SAC in both real-time and non-real-time settings.

\section{Background}
In Reinforcement Learning the world is split up into agent and environment. The agent is represented by a policy -- a state-conditioned action distribution, while the environment is represented by a Markov Decision Process (Def.~\ref{def:MDP}). Traditionally, the agent-environment interaction has been governed by the MDP framework. Here, however, we strictly use MDPs to represent the environment. The agent-environment interaction is instead described by different types of Markov Reward Processes (MRP), with the $ T\!B\!M\!R\!P $ (Def.~\ref{def:TBMRP}) behaving like the traditional interaction scheme.

\begin{definition}
\label{def:MDP}
A Markov Decision Process (MDP) is characterized by a tuple with

(1) state space $S$, \hspace{0.15cm} (2) action space $A$, \hspace{0.15cm}
(3) initial state distribution $\mu: S \to \mathbb R$, \\
(4) transition distribution {$p: S \times S \times A \to \mathbb R$}, \hspace{0.15cm}
(5) reward function $r: S \times A \to \mathbb R$.
\end{definition}

An agent-environment system can be condensed into a Markov Reward Process $(S, \mu, \kappa, \bar r)$ consisting of a Markov process $(S, \mu, \kappa)$ and a state-reward function $\bar r$. The Markov process induces a sequence of states $(s_t)_{t \in \mathbb N}$ and, together with $\bar r$, a sequence of rewards {$(r_t)_{t \in \mathbb N} = (\bar r(s_t))_{t \in \mathbb N}$}.

As usual, the objective is to find a policy that maximizes the expected sum of rewards. In practice, rewards can be discounted and augmented to guarantee convergence, reduce variance and encourage exploration. However, when evaluating the performance of an agent, we will always use the undiscounted sum of rewards.

\subsection{Turn-Based Reinforcement Learning}

\begin{wrapfigure}{r}{0.13\textwidth}
  \begin{center}
    \includegraphics[width=0.06 \textwidth, trim=-10 20 10 60]{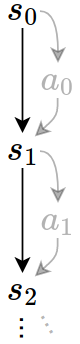}
  \end{center}
 \caption{$ T\!B\!M\!R\!P $}
 \label{fig:TBMRP}
\end{wrapfigure}

Usually considered part of the standard Reinforcement Learning framework is the turn-based scheme in which agent and environment interact. We call this interaction scheme Turn-Based Markov Reward Process.

\begin{definition} \label{def:TBMRP}
% A Turn-Based Markov Reward Process combines a MDP $E$ and a policy $\pi$ to form a MRP $(S, \mu, \kappa, \bar r) = T\!B\!M\!R\!P(E, \pi)$ with
A Turn-Based Markov Reward Process $(S, \mu, \kappa, \bar r) = T\!B\!M\!R\!P(E, \pi)$ combines a Markov Decision Process $E = (S, A, \mu, p, r)$ with a policy $\pi$, such that
\begin{equation} \label{mdp_kernel}
\medmuskip=0mu
\thinmuskip=0mu
\thickmuskip=0mu
%  \kappa^\pi_E: S \to \mathcal P(S) \quad \text{and} \quad
\kappa(s_{t+1} | s_t) = \int_A p(s_{t+1} | s_t, a) \pi(a | s_t) \ d a
\quad \text{and} \quad \bar r(s_t) = \int_A r(s_t, a) \pi(a | s_t) \ d a.
\
% \footnote{We write $\pi(a | s)$ instead of $\pi(s)(a)$ throughout the paper.} 
\end{equation}
% The state space and initial distribution of the Markov process are the same as in $E$.
\end{definition}

We say the interaction is turn-based, because the environment pauses while the agent selects an action and the agent pauses until it receives a new observation from the environment. This is illustrated in Figure~\ref{fig:TBDP}.
An action selected in a certain state is paired up again with that same state to induce the next. The state does not change during the action selection process.

\section{Real-Time Reinforcement Learning}

\begin{wrapfigure}{r}{0.13\textwidth}
  \begin{center}
    \includegraphics[width=0.07 \textwidth, trim=-5 1.1cm 5 2.4cm]{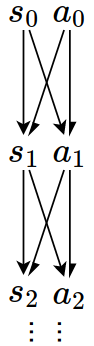}  % trim=left bottom right top
  \end{center}
  \caption{$ R\!T\!M\!R\!P $}
  \label{fig:RTMRP}
  \vspace{-0.5cm}
\end{wrapfigure}

In contrast to the conventional, turn-based interaction scheme, we propose an alternative, real-time interaction framework in which states and actions evolve simultaneously.
Here, agent and environment step in unison to produce new state-action pairs $\x_{t+1} = (\itup{s_{t+1}, a_{t+1}})$
% \footnote{Real-time states $(\itup{s_t, a_t}) \in \augm X$ have a grey background color.}
from old state-action pairs $\x_t = (\itup{s_t, a_t})$ as illustrated in Figures~\ref{fig:RTDP}~and~\ref{fig:RTMRP}.
\begin{definition} \label{def:RTMRP}
A Real-Time Markov Reward Process $(\X, \augm \mu, \augm \kappa, \bar{\augm r}) = R\!T\!M\!R\!P(E, \augm \pi)$  combines a Markov Decision Process $E = (S, A, \mu, p, r)$ with a policy $\pi$, such that
\begin{equation}
\medmuskip=0mu
\thinmuskip=0mu
\thickmuskip=0mu
% \augm \kappa(\x_{t+1} | \x_t) = \augm \kappa((s_{t+1}, a_{t+1}) | \x_t) = p(s_{t+1} | \x_t) \ \augm \pi(a_{t+1} | \x_t) 
\augm \kappa(\itup{s_{t+1}, a_{t+1}} | \itup{s_t, a_t})
= p(s_{t+1} | s_t, a_t) \ \augm \pi(a_{t+1} | s_t, a_t) 
% \quad \text{and} \quad \bar{\augm r}(\x) = r(s, u).
\quad \text{and} \quad \bar{\augm r}(\itup{s_t, a_t}) = r(s_t, a_t).
\end{equation}
The system state space is $\X = S \times A$. The initial action $a_0$ can be set to some fixed value, i.e. $\augm \mu(\itup{s_0, a_0}) = \mu(s_0) \ \delta(a_0-c)$.%
\footnote{$\delta$ is the Dirac delta distribution. If $y \sim \delta(\cdot - x)$ then $y=x$ with probability one.}
\end{definition}

Note that we introduced a new policy $\augm \pi$ that takes state-action pairs instead of just states. That is because the system state $\augm x = \itup{(s, a)}$ is now a state-action pair and $s$ alone is not a sufficient statistic of the future of the stochastic process anymore.

\subsection{The real-time framework is made for back-to-back action selection}

\vspace{-0.2cm}

In the real-time framework, the agent has exactly one timestep to select an action. If an agent takes longer that its policy would have to be broken up into stages that take less than one timestep to evaluate. On the other hand, if an agent takes less than one timestep to select an action, the real-time framework will delay applying the action until the next observation is made. The optimal case is when an agent, immediately upon finishing selecting an action, observes the next state and starts computing the next action. This continuous, back-to-back action selection is ideal in that it allows the agent to update its actions the quickest and no delay is introduced through the real-time framework.
% e.g. if we can compute an action in 1ms we should do so a 1000 times per second.

To achieve back-to-back action selection, it might be necessary to match timestep size to the policy evaluation time. With current algorithms, reducing timestep size might lead to worse performance. Recently, however, progress has been made towards timestep agnostic methods \citep{tallec2019making}. We believe back-to-back action selection is an achievable goal and we demonstrate here that the real-time framework is effective even if we are not able to tune timestep size (Section~\ref{SectionExperiments}).

% A different perspective on RTRL is to view $\pi$ as the update equation of a discretization of a differential equation for $u$ (see Appendix \TODO{...}).

\vspace{-0.2cm}

\subsection{Real-time interaction can be expressed within the turn-based framework}
It is possible to express real-time interaction within the standard, turn-based framework, which allows us to reconnect the real-time framework to the vast body of work in RL. Specifically, we are trying to find an augmented environment $R\!T\!M\!D\!P(E)$ that behaves the same with turn-based interaction as would $E$ with real-time interaction.

In the real-time framework the agent communicates its action to the environment via the state. However, in the traditional, turn-based framework, only the environment can directly influence the state. We therefore need to deterministically "pass through" the action to the next state by augmenting the transition function.
The $R\!T\!M\!D\!P$ has two types of actions, (1) the actions $\augm a_t$ emitted by the policy and (2) the action component $a_t$ of the state $\augm x_t = (\itup{s_t, a_t})$, where $a_t = \augm a_{t-1}$ with probability one.

\begin{definition} \label{def:RTMDP}
A Real-Time Markov Decision Process $(\X, A, \augm \mu, \augm p, \augm r) = R\!T\!M\!D\!P(E)$ augments another Markov Decision Process $E = (S, A, \mu, p, r)$, such that

(1) state space $\X = S \times A$, \hspace{0.1cm} (2) action space is $A$, \\
(3) initial state distribution $
\augm \mu(\x_0) = 
\augm \mu(\itup{s_0, a_0}) = \mu(s_0) \ \delta(a_0 - c)$, \\  % usually this is called \mu
(4) transition distribution $
\augm p(\x_{t+1}|\x_t, \augm a_t) =
\augm p(\itup{s_{t+1}, a_{t+1}} | \itup{s_t, a_t}, \augm a_t)
= p(s_{t+1} | s_t, a_t) \ \delta(a_{t+1} - \augm a_t)
$ \\   % TODO: this line is the most complex with flattened tuples, perhaps add footnote?
(5) reward function $
\augm r(\x_t, {\color{gray} \augm a_t}) = 
\augm r(\itup{s_t, a_t}, {\color{gray} \augm a_t}) = r(s_t, a_t)
$.
\begin{flushright}
\vspace{-0.6cm}
\href{https://github.com/rmst/rtrl/blob/master/rtrl/wrappers.py#L7}{\tiny{(tap to see code)}}
\end{flushright}
\end{definition}
\begin{restatable}{theorem}{TheoremRttb} \!\! \footnote{All proofs are in Appendix~\ref{proofs}.}
A policy $\augm \pi: A \times \X \to \mathbb R$ interacting with $R\!T\!M\!D\!P(E)$ in the conventional, turn-based manner gives rise to the same Markov Reward Process as $\augm \pi$ interacting with $E$ in real-time, i.e.
\begin{equation}
    R\!T\!M\!R\!P(E, \augm \pi) = T\!B\!M\!R\!P(R\!T\!M\!D\!P(E), \augm \pi).
    % E \otimes \augm \pi = R\!T\!M\!D\!P(E) \odot \augm \pi.
\end{equation}
%
% {\hspace{0.1cm} \scriptsize (Proof in Appendix \ref{proofs})}
\end{restatable}
% \begin{proof} See Appendix \ref{AppLemmaRttb}. \end{proof}
\vspace{-0.2cm}
Interestingly, the RTMDP is equivalent to a 1-step constant delay MDP (\cite{Walsh2008LearningAP}). However, we believe the different intuitions behind both of them warrant the different names: The constant delay MDP is trying to model external action and observation delays whereas the RTMDP is modelling the time it takes to select an action. The connection makes sense, though: In a framework where the action selection is assumed to be instantaneous, we can apply a delay to account for the fact that the action selection was not instantaneous after all.

\vspace{-0.2cm}

\subsection{Turn-based interaction can be expressed within the real-time framework}
It is also possible to define an augmentation $T\!B\!M\!D\!P(E)$ that allows us to express turn-based environments (e.g. Chess, Go) within the real-time framework (Def.~\ref{def:TBMDP} in the Appendix).
By assigning separate timesteps to agent and environment, we can allow the agent to act while the environment pauses. More specifically, we add a binary variable $b$ to the state to keep track of whether it is the environment's or the agent's turn. While $b$ inverts at every timestep, the underlying environment only advances every other timestep.

\begin{restatable}{theorem}{TheoremTbrt}
A policy $\augm \pi(\augm a | {s, b, a}) = \pi(\augm a|s)$ interacting with $T\!B\!M\!D\!P(E)$ in real time, gives rise to a Markov Reward Process that contains (Def.~\ref{def:MRP-contains}) the MRP resulting from $\pi$ interacting with $E$ in the conventional, turn-based manner, i.e.
\begin{equation}
    T\!B\!M\!R\!P(E, \pi) \propto R\!T\!M\!R\!P(T\!B\!M\!D\!P(E), \augm \pi)
\end{equation}
\end{restatable}

\vspace{-0.2cm}

As a result, not only can we use conventional algorithms in the real-time framework but we can use algorithms built on the real-time framework for all turn-based problems.

\newpage

\section{Reinforcement Learning in Real-Time Markov Decision Processes} \label{SectionSacInRt}
Having established the RTMDP as a compatibility layer between conventional RL and RTRL, we can now look how existing theory changes when moving from an environment $E$ to $R\!T\!M\!D\!P(E)$.

Since most RL methods assume that the environment's dynamics are completely unknown, they will not be able to make use of the fact that we precisely know part of the dynamics of RTMDP. Specifically they will have to learn from data, the effects of the "feed-through" mechanism which could lead to much slower learning and worse performance when applied to an environment $R\!T\!M\!D\!P(E)$ instead of $E$.
This could especially hurt the performance of off-policy algorithms which have been among the most successful RL methods to date \citep{mnih2015human, haarnoja2018soft}.
Most off-policy methods make use of the action-value function.
\begin{definition}
The action value function $q_E^\pi$ for an environment $E=(S, A, \mu, p, r)$ and a policy $\pi$ can be recursively defined as
\begin{equation} \label{EqQ}
    q_E^\pi(s_t, a_t) = r(s_t, a_t) + \E_{s_{t+1} \sim p(\cdot | s_t, a_t)}[ \E_{a_{t+1} \sim \pi(\cdot | s_{t+1})}[ q_E^\pi(s_{t+1}, a_{t+1})]]
\end{equation}
\end{definition}
When this identity is used to train an action-value estimator, the transition $s_t, a_t, s_{t+1}$ can be sampled from a replay memory containing off-policy experience while the next action $a_{t+1}$ is sampled from the policy $\pi$.
\begin{restatable}{lemma}{LemmaActionValue}
\label{lemma:statevalue}
% \begin{lemma}
In a Real-Time Markov Decision Process for the action-value function we have
\begin{equation}
\medmuskip=0mu
\thinmuskip=0mu
\thickmuskip=0mu
% q_\text{\textit{RTMDP(E)}}^{\augm \pi}(\augm x_t, \augm a_t) = 
q_\text{\textit{RTMDP(E)}}^{\augm \pi}(\itup{s_t, a_t}, \augm a_t) =
r(s_t, a_t) + \E_{s_{t+1} \sim p(\cdot | s_t, a_t)}[ \E_{\augm a_{t+1} \sim \augm \pi(\cdot | \itup{s_{t+1}, \augm a_t})} [q_\text{\textit{RTMDP(E)}}^{\augm \pi}(\itup{s_{t+1}, \augm a_t}, \augm a_{t+1})]]
\end{equation}
% 
% \end{lemma}
\end{restatable}
Note that the action $\augm a_t$ does not affect the reward nor the next state. The only thing that $\augm a_t$ does affect is $a_{t+1}$ which, in turn, only in the next timestep will affect $r(s_{t+1}, a_{t+1})$ and $s_{t+2}$. To learn the effect of an action on $E$ (specifically the future rewards), we now have to perform two updates where previously we only had to perform one. We will investigate experimentally the effect of this on the off-policy Soft Actor-Critic algorithm \citep{haarnoja2018soft} in Section \ref{SacStruggles}.

\subsection{Learning the state-value function off-policy} \label{SectionOffPolicyValue}
The state-value function can usually not be used in the same way as the action-value function for off-policy learning.
\begin{definition}
The state-value function $v_E^\pi$ for an environment $E=(S, A, \mu, p, r)$ and a policy $\pi$ is
\begin{equation}
v_E^\pi(s_t) = \E_{a_t \sim \pi(\cdot | s_t)}[r(s_t, a_t) + \E_{s_{t+1} \sim p(\cdot | s_t, a_t)}[v_E^\pi(s_{t+1})]]
\end{equation}
\end{definition}
The definition shows that the expectation over the action is taken \emph{before} the expectation over the next state. When using this identity to train a state-value estimator, we cannot simply change the action distribution to allow for off-policy learning since we have no way of resampling the next state.
\begin{restatable}{lemma}{LemmaStateValue}
% \begin{lemma}
In a Real-Time Markov Decision Process for the state-value function we have
\begin{equation} \label{EqV}
% v_\text{\textit{RTMDP(E)}}^{\augm \pi}(\augm x_t) =
v_\text{\textit{RTMDP(E)}}^{\augm \pi}(\itup{s_t, a_t}) = r(s_t, a_t) + \E_{s_{t+1} \sim p(\cdot | s_t, a_t)} [\E_{\augm a_t \sim \augm \pi(\cdot | \itup{s_{t}, a_t})}[ v_\text{\textit{RTMDP(E)}}^{\augm \pi}(\itup{s_{t+1}, \augm a_{t}}) ]].
\end{equation}
%
% \end{lemma}
\end{restatable}
Here, $s_t, a_t, s_{t+1}$ is always a valid transition no matter what action $\augm a_t$ is selected. Therefore, when using the real-time framework, we can use the value function for off-policy learning.
% This allows for partial simulation in the agent.
Since Equation~\ref{EqV} is the same as Equation~\ref{EqQ} (except for the policy inputs), we can use the state-value function where previously the action-value function was used without having to learn the dynamics of the $R\!T\!M\!D\!P$ from data since they have already been applied to Equation \ref{EqV}.
\subsection{Partial simulation}
The off-policy learning procedure described in the previous section can be applied more generally. Whenever parts of the agent-environment system are known and (temporarily) independent of the remaining system, they can be used to generate synthetic experience. More precisely, transitions\hide{ (or longer sub-trajectories) } with a start state $s = (w, z)$ can be generated according to the true transition kernel $\kappa(s'|s)$ by simulating the known part of the transition ($w \to w'$) and using a stored sample for the unknown part of the transition ($z \to z'$). This is only possible if the transition kernel factorizes~as~${\kappa(w', z'| s) = \kappa_\text{known}(w'|s) \ \kappa_\text{unknown}(z'|s)}$. Hindsight Experience Replay \citep{andrychowicz2017hindsight} can be seen as another example of partial simulation. There, the goal part of the state evolves independently of the rest which allows for changing the goal in hindsight. In the next section, we use the same partial simulation principle to compute the gradient of the policy loss.

\section{Real-Time Actor-Critic (RTAC)}
Actor-Critic algorithms \citep{konda2000actor} formulate the RL problem as bi-level optimization where the critic evaluates the actor as accurately as possible while the actor tries to improve its evaluation by the critic. \citet{silver2014deterministic} showed that it is possible to reparameterize the actor evaluation and directly compute the pathwise derivative from the critic with respect to the actor parameters and thus telling the actor how to improve. \citet{heess2015learning} extended that to stochastic policies and \cite{haarnoja2018soft} further extended it to the maximum entropy objective to create Soft Actor-Critic (SAC) which RTAC is going to be based on and compared against.

In SAC a parameterized policy $\pi$ (the actor)
% \footnote{We are referring to both the policy as well as its parameters as $\pi$.} 
is optimized to minimize the KL-divergence between itself and the exponential of an approximation of the action-value function $q$ (the critic) normalized by $Z$ (where $Z$ is unknown but irrelevant to the gradient) giving rise to the policy loss
\begin{equation}\label{ActorLoss}
    L^\text{SAC}_{E, \pi} = \E_{s_t \sim D} \KL(\pi(\cdot | s_t) || \exp (\tfrac 1 \alpha q(s_t, \cdot)) / Z(s_t))
    \footnote{$\alpha$ is a temperature hyperparameter. For $\alpha \to 0$, the maximum entropy objective reduces to the traditional objective. To compare with the hyperparameters table we have $\alpha = \frac{\text{entropy scale}}{\text{reward scale}}$.}
\end{equation}
where $D$ is a uniform distribution over the replay memory containing past states, actions and rewards. The action-value function itself is optimized to fit Equation \ref{EqQ} presented in the previous section (augmented with an entropy term). We can thus expect SAC to perform worse in RTMDPs.

In order to create an algorithm better suited for the real-time setting we propose to use a state-value function approximator ${\augm v}$ as the critic instead, that will give rise to the same policy gradient.
\begin{restatable}{prop}{PropRtac}
% \begin{prop}
The following policy loss based on the state-value function
\begin{equation}\label{RtacActorLoss}
    L^\text{RTAC}_{R\!T\!M\!D\!P(E), \augm \pi} = \E_{(s_t, a_t) \sim D} \E_{s_{t+1}\sim p(\cdot | s_t, a_t)} \KL(\augm \pi(\cdot | \itup{s_t, a_t}) || \exp (\tfrac 1 \alpha \gamma {\augm v}(\itup{s_{t+1}, \cdot})) / Z(s_{t+1}))
\end{equation}
has the same policy gradient as {\small $L^\text{SAC}_{R\!T\!M\!D\!P(E), \augm \pi}$}, i.e.
\begin{equation}
    \nabla_{\augm \pi} L^\text{RTAC}_{R\!T\!M\!D\!P(E), \augm \pi}
    = \nabla_{\augm \pi} L^\text{SAC}_{R\!T\!M\!D\!P(E), \augm \pi}
\end{equation}
% We need an extra $\gamma$ in the exponential to account for the discounting of the value function.
% \end{prop}
\end{restatable}
The value function itself is trained off-policy according to the procedure described in Section \ref{SectionOffPolicyValue} to fit an augmented version of Equation \ref{EqV}, specifically
\begin{equation} \label{EqVtarget}
\augm v_\text{target}
= r(s_t, a_t) + \E_{s_{t+1} \sim p(\cdot | s_t, a_t)} [\E_{\augm a_t \sim \augm \pi(\cdot | s_{t}, a_t)}[\bar {\augm v}_{\bar \theta}((s_{t+1}, \augm a_{t})) - \alpha \log({\augm \pi}(\augm a_t | s_{t}, a_t))]].
\end{equation}
% $$v_\text{target} = \augm r_t +  \gamma  \E_{a_{t+1} \sim \pi(\x_t)}[v_{\bar\theta}((s_{t+1}, a_{t+1}))] $$
Therefore, for the value loss, we have
\begin{equation}
    L^\text{RTAC}_{R\!T\!M\!D\!P(E), {\augm v}}
    = \E_{(\x_t, \augm r_t, s_{t+1}) \sim D}[({\augm v}(\x_t) - \augm v_\text{target})^2]
\end{equation}
% $$L_v(\theta) = \E_{(\x_t, \augm r_t, s_{t+1}) \sim D}[(v_\theta(\x_t) - v_\text{target}(\x_t))^2]$$

\subsection{Merging actor and critic}
Using the state-value function as the critic has another advantage: When evaluated at the same timestep, the critic does not depend on the actor's output anymore and we are therefore able to use a single neural network to represent both the actor and the critic. Merging actor and critic makes it necessary to trade off between the value function and policy loss. Therefore, we introduce an additional hyperparameter $\beta$.
\begin{equation}
L(\theta)
% = L((\augm \pi, \tilde{\augm v}))
= \beta L^\text{RTAC}_{R\!T\!M\!D\!P(E), \augm \pi_\theta} + (1-\beta) L^\text{RTAC}_{R\!T\!M\!D\!P(E), {\augm v_\theta}}
\label{loss}
\end{equation}
Merging actor and critic could speed up learning and even improve generalization, but could also lead to greater instability. We compare RTAC with both merged and separate actor and critic networks in Section~\ref{SectionExperiments}.

\subsection{Stabilizing learning}

\hide{
\begin{wrapfigure}{l}{0.4\textwidth}
% \vspace{-0.5cm}
\vspace{-0.51cm}
\begin{algorithm}[H]
\SetCustomAlgoRuledWidth{0.45\textwidth}
% \SetAlgoLined
% \KwResult{Write here the result }
 Initialize parameter vectors $\theta, \bar\theta$ \\
 \For{each iteration}{
  \For{each environment step}{
    $a_{t+1} \sim \pi(\cdot | s_t, a_t)$ \\
    $s_{t+1} \sim p(\cdot | s_t, a_t)$ \\
    $D \leftarrow D \cup \{(s_t, a_t, r_t, s_{t+1})\}$ \\
  }
  \For{each gradient step}{
    $\theta \leftarrow \theta + \lambda \nabla_\theta L(\theta) \quad $ Eqn. \ref{loss} \\
    $\bar\theta \leftarrow \tau \theta + (1-\tau) \bar\theta$
  }
 }
\begin{flushright}
\vspace{-0.3cm}
\href{https://github.com/rmst/rtrl/blob/master/rtrl/rtac.py}{\tiny{(tap to see code)}}
\end{flushright}
 \caption{Real-Time Actor-Critic}
\end{algorithm}
\vspace{-0.5cm}
\end{wrapfigure}
}
% Actor-Critic algorithms, like Generative Adversarial Networks \citep{goodfellow2014generative} are bi-level optimization problems \citep{pfau2016connecting} where one parametric model (e.g. policy or generator) is optimized with respect to another (e.g. critic or discriminator). This makes the optimization difficult and unstable.

% In contrast to supervised learning, in bilevel optimization the gradient vector field that determines the parameter updates in (simultaneous) stochastic gradient descend is both non-conservative \citep{ferenc_noncon} and changes over the course of optimization.
Actor-Critic algorithms are known to be unstable during training. We use a number of techniques that help make training more stable. Most notably we use Pop-Art output normalization \citep{van2016learning} to normalize the value targets. This is necessary if $v$ and $\pi$ are represented using an overlapping set of parameters. Since the scale of the error gradients of the value loss is highly non-stationary it is hard to find a good trade-off between policy and value loss ($\beta$). If $v$ and $\pi$ are separate, Pop-Art matters less, but still improves performance both in SAC as well as in RTAC.

Another difficulty are the recursive value function targets. Since we try to maximize the value function, overestimation errors in the value function approximator are amplified and recursively used as target values in the following optimization steps. As introduced by \citet{fujimoto2018addressing} and like SAC, we will use two value function approximators and take their minimum when computing the target values to reduce value overestimation, i.e. $\bar {\augm v}_{\bar\theta}(\cdot) =  \min_{i\in\{1, 2\}} \ \augm v_{\bar\theta, i}(\cdot)$.

Lastly, to further stabilize the recursive value function estimation, we use target networks that slowly track the weights of the network \citep{mnih2015human, lillicrap2015continuous}, i.e. $\bar\theta \leftarrow \tau \theta + (1-\tau) \bar\theta$. The tracking weights $\bar \theta$ are then used to compute $\augm v_\text{target}$ in Equation \ref{EqVtarget}.

\vspace{-.15cm}

\section{Experiments} \label{SectionExperiments}
\vspace{-.15cm}

We compare Real-Time Actor-Critic to Soft Actor-Critic \citep{haarnoja2018soft} on several OpenAI-Gym/MuJoCo benchmark environments \citep{gym, todorovET12} as well as on two Avenue autonomous driving environments with visual observations \citep{ibrahim2019avenue}.

The SAC agents used for the results here, include both an action-value and a state-value function approximator and use a fixed entropy scale $\alpha$ (as in \citet{haarnoja2018soft}). In the code accompanying this paper we dropped the state-value function approximator since it had no impact on the results (as done and observed in \citet{haarnoja2018soft2}).
For a comparison to other algorithms such as DDPG, PPO and TD3 also see \citet{haarnoja2018soft, haarnoja2018soft2}.

% \paragraph{Implementation}
To make the comparison between the two algorithms as fair as possible, we also use output normalization in SAC which improves performance on all tasks (see Figure~\ref{SacNoNorm} in Appendix~\ref{AdditionalExperiments} for a comparison between normalized and unnormalized SAC).
% We both show the performance of SAC in real-time and non-real-time environments.
% The performance of our SAC implementation in the non-real-time environments matches \citet{haarnoja2018soft, haarnoja2018soft2} almost exactly.
Both SAC and RTAC are performing a single optimization step at every timestep in the environment starting after the first $10000$ timesteps of collecting experience based on the initial random policy.
The hyperparameters used can be found in Table~\ref{Hyperparameters}.

% \paragraph{Figures} All figures show return trends over several runs. For each run, the test return is computed each $20000$ timesteps as the average return over $100000$ timesteps using a deterministic policy. For each run the test returns are then smoothed with window size $0.1 \times$number of test returns per run. The return trends show the mean over all runs of the smoothed test returns whereas the shaded region is the $95\%$ confidence interval assuming independently, normally distributed data points with unknown mean and variance. 

\vspace{-.15cm}

\subsection{SAC in Real-Time Markov Decision Processes} \label{SacStruggles}
\vspace{-.15cm}

When comparing the return trends of SAC in turn-based environments $E$ against SAC in real-time environments $R\!T\!M\!D\!P(E)$, the performance of SAC deteriorates. This seems to confirm our hypothesis that having to learn the dynamics of the augmented environment from data impedes action-value function approximation (as hypothesized in Section~\ref{SectionSacInRt}). 

\begin{figure}[H] \label{RtVsTb}
  \centering
  \includegraphics[trim=0.7cm 0 0 0, width=1.02\linewidth]{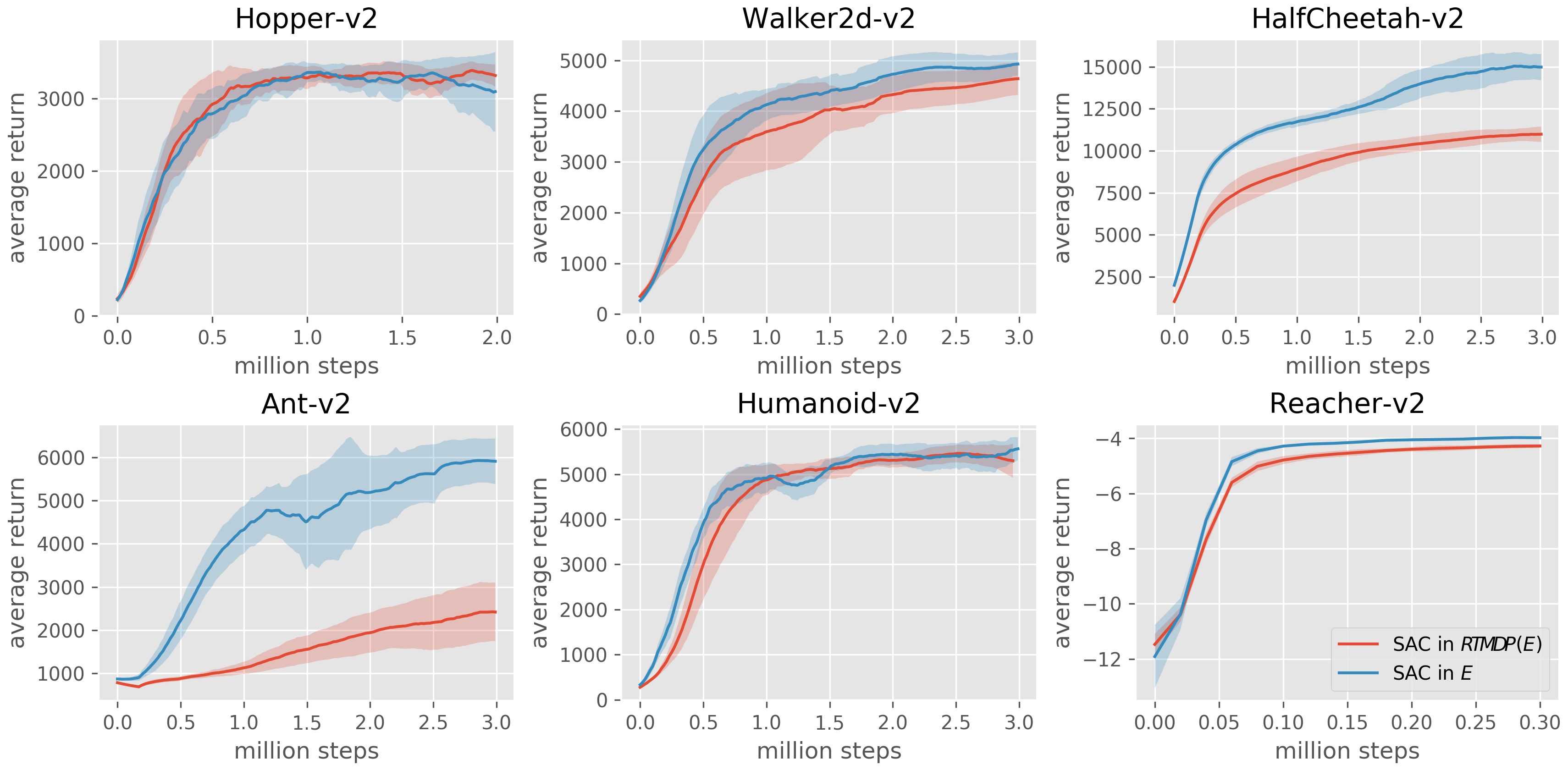}
  \caption{Return trends for SAC in turn-based environments $E$ and real-time environments $R\!T\!M\!D\!P(E)$. Mean and $95\%$ confidence interval are computed over eight training runs per environment.}
\end{figure}

\subsection{RTAC and SAC on MuJoCo in real time} \label{RtacWorks}

Figure~\ref{RtacSacComp} shows a comparison between RTAC and SAC in real-time versions of the benchmark environments. We can see that RTAC learns much faster and achieves higher returns than SAC in $R\!T\!M\!D\!P(E)$. This makes sense as it does not have to learn from data the "pass-through" behavior of the RTMDP. We show RTAC with separate neural networks for the policy and value components showing that a big part of RTAC's advantage over SAC is its value function update. However, the fact that policy and value function networks can be merged 
further improves RTAC's performance as the plots suggest. Note that RTAC is always in $R\!T\!M\!D\!P(E)$, therefore we do not explicitly state it again.

RTAC is even outperforming SAC in $E$ (when SAC is allowed to act without real-time constraints) in four out of six environments including the two hardest ones - Ant and Humanoid - with largest state and action space (Figure~\ref{RtacSacTb}). We theorize this is possible due to the merged actor and critic networks used in RTAC. It is important to note however, that for RTAC with merged actor and critic networks output normalization is critical (Figure~\ref{RtacNoNorm}). 

\vspace{-0.3cm}

\begin{figure}[H]
  \centering
  \includegraphics[trim=0.7cm 0 0 0, width=1.0\linewidth]{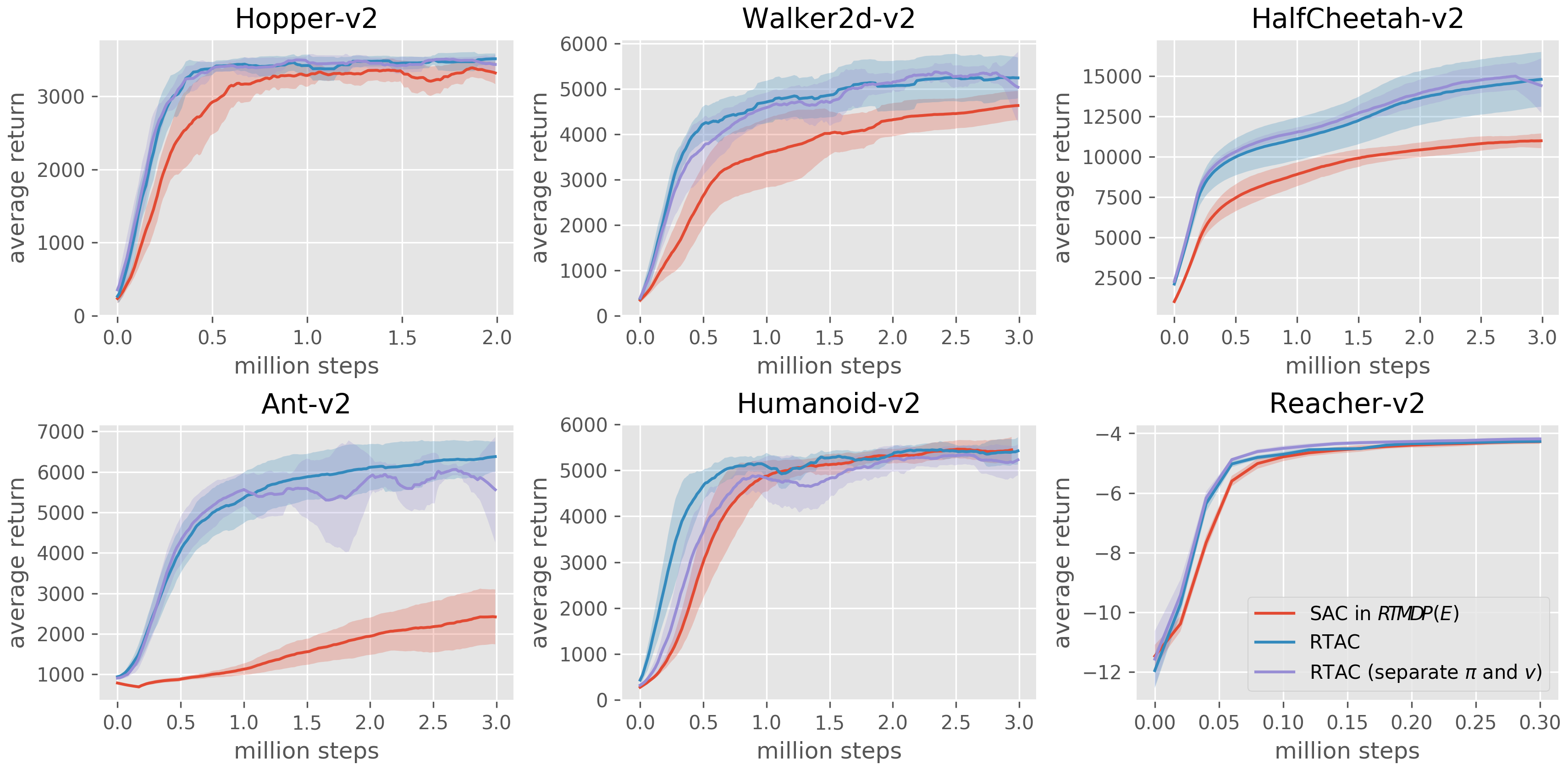}
  \caption{Comparison between RTAC and SAC in RTMDP versions of the benchmark environments.
  Mean and $95\%$ confidence interval are computed over eight training runs per environment.
    % For reference we also show the performance of SAC in the non-real time versions of the environments.
  }
  \label{RtacSacComp}
\end{figure}

\subsection{Autonomous driving task}
In addition to the MuJoCo environments, we have also tested RTAC and SAC on an autonomous driving task using the Avenue simulator \citep{ibrahim2019avenue}.
Avenue is a game-engine-based simulator where the agent controls a car. In the task shown here, the agent has to stay on the road and possibly steer around pedestrians. The observations are single image (256x64 grayscale pixels) and the car's velocity. The actions are continuous and two dimensional, representing steering angle and gas-brake. The agent is rewarded proportionally to the car's velocity in the direction of the road and negatively rewarded when making contact with a pedestrian or another car. In addition, episodes are terminated when leaving the road or colliding with any objects or pedestrians.

\begin{figure}[H]
    \centering
    \includegraphics[width=\linewidth]{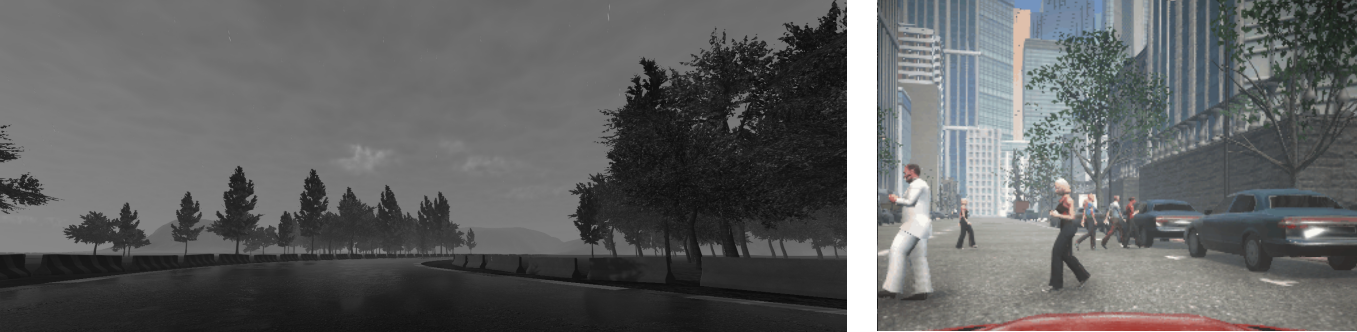}
    \caption{Left: Agent's view in \texttt{RaceSolo}. Right: Passenger view in \texttt{CityPedestrians}.}
    \label{fig:avcollage}
\end{figure}

\begin{figure}[H]
  \includegraphics[width=\linewidth]{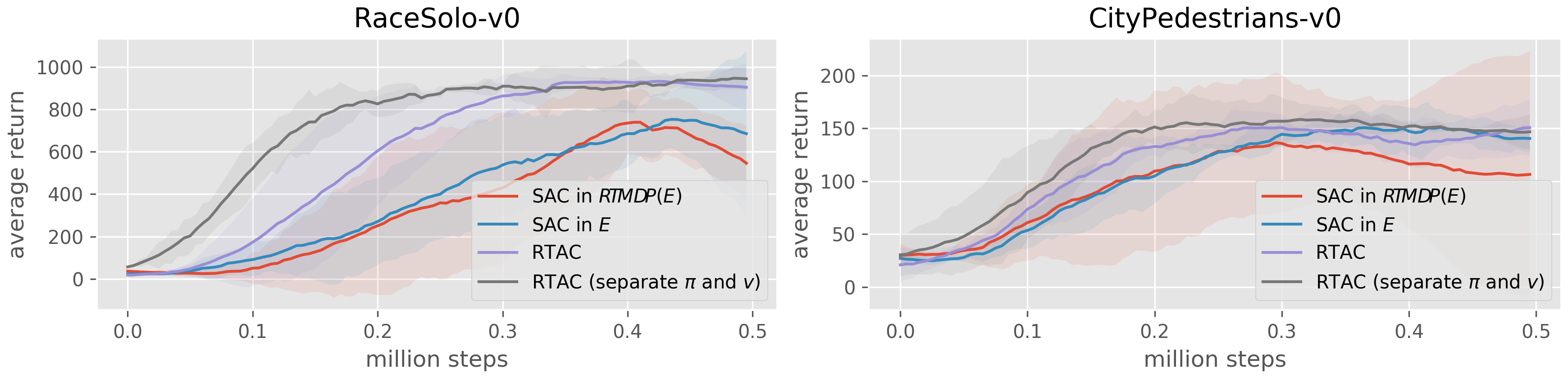}
  \caption{Comparison between RTAC and SAC in RTMDP versions of the autonomous driving tasks. We can see that RTAC under real-time constraints outperforms SAC even without real-time constraints. 
  Mean and $95\%$ confidence interval are computed over four training runs per environment.
  \label{RtacSacAv}}
\end{figure}

The hyperparameters used for the autonomous driving task are largely the same as for the MuJoCo tasks, however we used a lower entropy reward scale ($0.05$) and lower learning rate ($0.0002$). We used convolutional neural networks with four layers of convolutions with filter sizes $(8, 4, 4, 4)$, strides $(2, 2, 2, 1)$ and $(64, 64, 128, 128)$ channels. The convolutional layers are followed by two fully connected layers with $512$ units each.

\section{Related work}

\citet{travnik} noticed that the traditional MDP framework is ill suited for real-time problems. Other than our paper, however, no rigorous framework is proposed as an alternative, nor is any theoretical analysis provided.

\cite{Firoiu2018AtHS} applies a multi-step action delay to level the playing field between humans and artificial agents on the ALE (Atari) benchmark
\TODO{cite}
However, it does not address the problems arising from the turn-based MDP framework or recognizes the significance and consequences of the one-step action delay.

Similar to RTAC, NAF \citep{gu2016continuous} is able to do continuous control with a single neural network. However, it is requiring the action-value function to be quadratic in the action (and thus possible to optimize in closed form). This assumption is quite restrictive and could not outperform more general methods such as DDPG.

In SVG(1) \citep{heess2015learning} a differentiable transition model is used to compute the path-wise derivative of the value function one timestep after the action selection. This is similar to what RTAC is doing when using the value function to compute the policy gradient. However, in RTAC, we use the actual differentiable dynamics of the RTMDP, i.e. "passing through" the action to the next state, and therefore we do not need to approximate the transition dynamics. At the same time, transitions for the underlying environment are not modelled at all and instead sampled which is only possible because the actions $\augm a_t$ in a RTMDP only start to influence the underlying environment at the next timestep.

\section{Discussion}
We have introduced a new framework for Reinforcement Learning, RTRL, in which agent and environment step in unison to create a sequence of state-action pairs. We connected RTRL to the conventional Reinforcement Learning framework through the RTMDP and investigated its effects in theory and practice. We predicted and confirmed experimentally that conventional off-policy algorithms would perform worse in real-time environments and then proposed a new actor-critic algorithm, RTAC, that not only avoids the problems of conventional off-policy methods with real-time interaction but also allows us to merge actor and critic which comes with an additional gain in performance. We showed that RTAC outperforms SAC on both a standard, low dimensional continuous control benchmark, as well as a high dimensional autonomous driving task.

\newpage

\subsubsection*{Acknowledgments}
We would like to thank Cyril Ibrahim for building and helping us with the Avenue simulator; Craig Quiter and Sherjil Ozair for insightful discussions about agent-environment interaction; Alex Piché,
Scott Fujimoto,
Bhairav Metha and
Jhelum Chakravorty,
for reading drafts of this paper and finally
Jose Gallego, Olexa Bilaniuk and many others at Mila that helped us on countless occasions online and offline.

This work was completed during a part-time internship at Element AI and was supported by the Open Philanthropy Project.

% \TODO{Replace arxiv references with conference proceedings where possible.}
\bibliography{main}
\bibliographystyle{icml2018.bst}

\newpage
\appendix

\section{Additional Experiments} \label{AdditionalExperiments}

\begin{figure}[H]
  \centering
  \includegraphics[trim=0.7cm 0 0 0, width=1.02\linewidth]{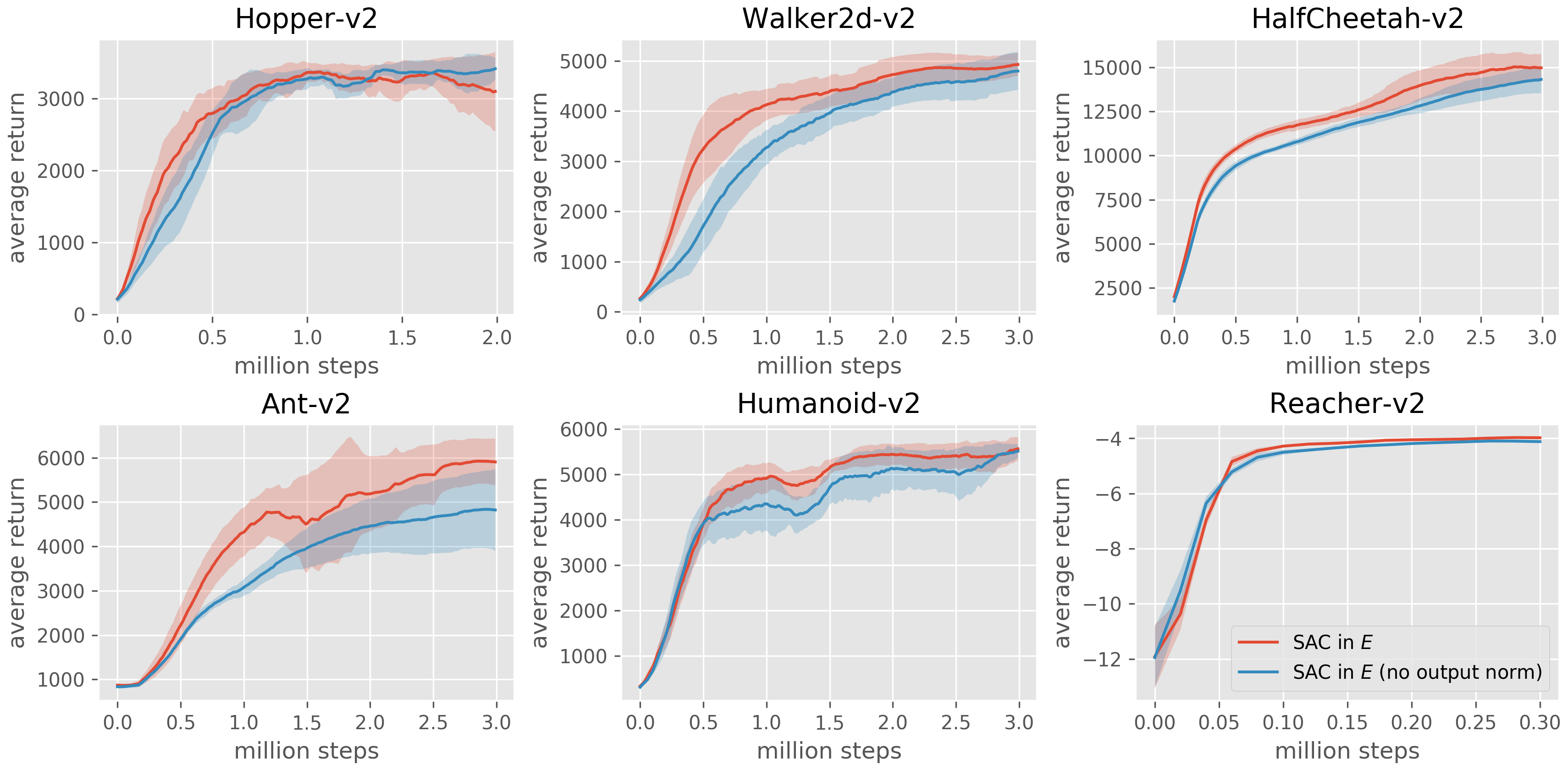}
  \caption{SAC with and without output normalization. SAC in $E$ (no output norm) corresponds to the canonical version presented in \cite{haarnoja2018soft}. Mean and $95\%$ confidence interval are computed over eight training runs per environment.}
  \label{SacNoNorm}
\end{figure}

\begin{figure}[H] 
  \centering
  \includegraphics[trim=0.7cm 0 0 0, width=1.02\linewidth]{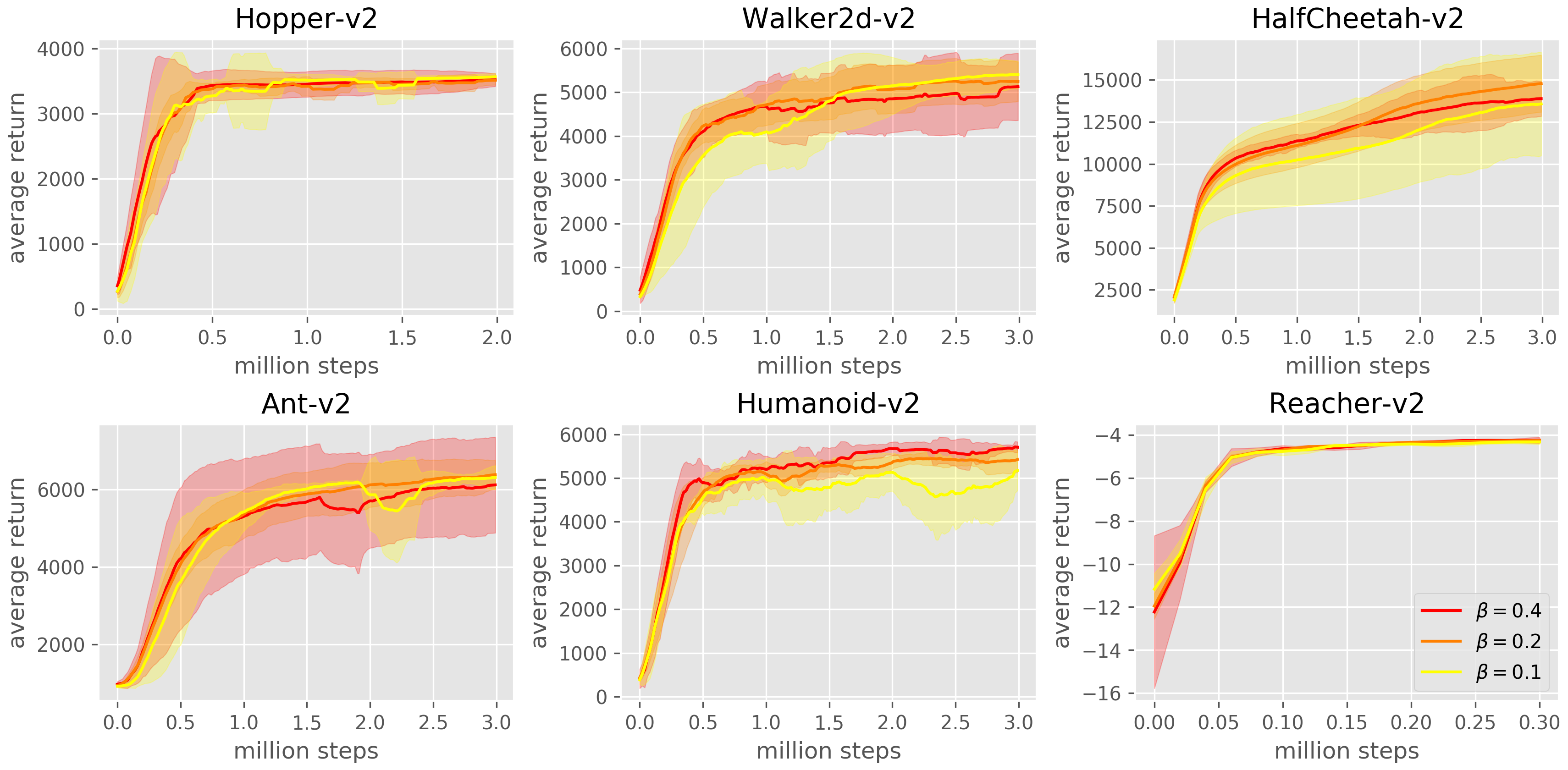}
  \caption{Comparison between different actor loss scales ($\beta$). Mean and $95\%$ confidence interval are computed over four training runs per environment.}
  \label{RtacBeta}
\end{figure}

\begin{figure}[H] 
  \centering
  \includegraphics[trim=0.7cm 0 0 0, width=1.02\linewidth]{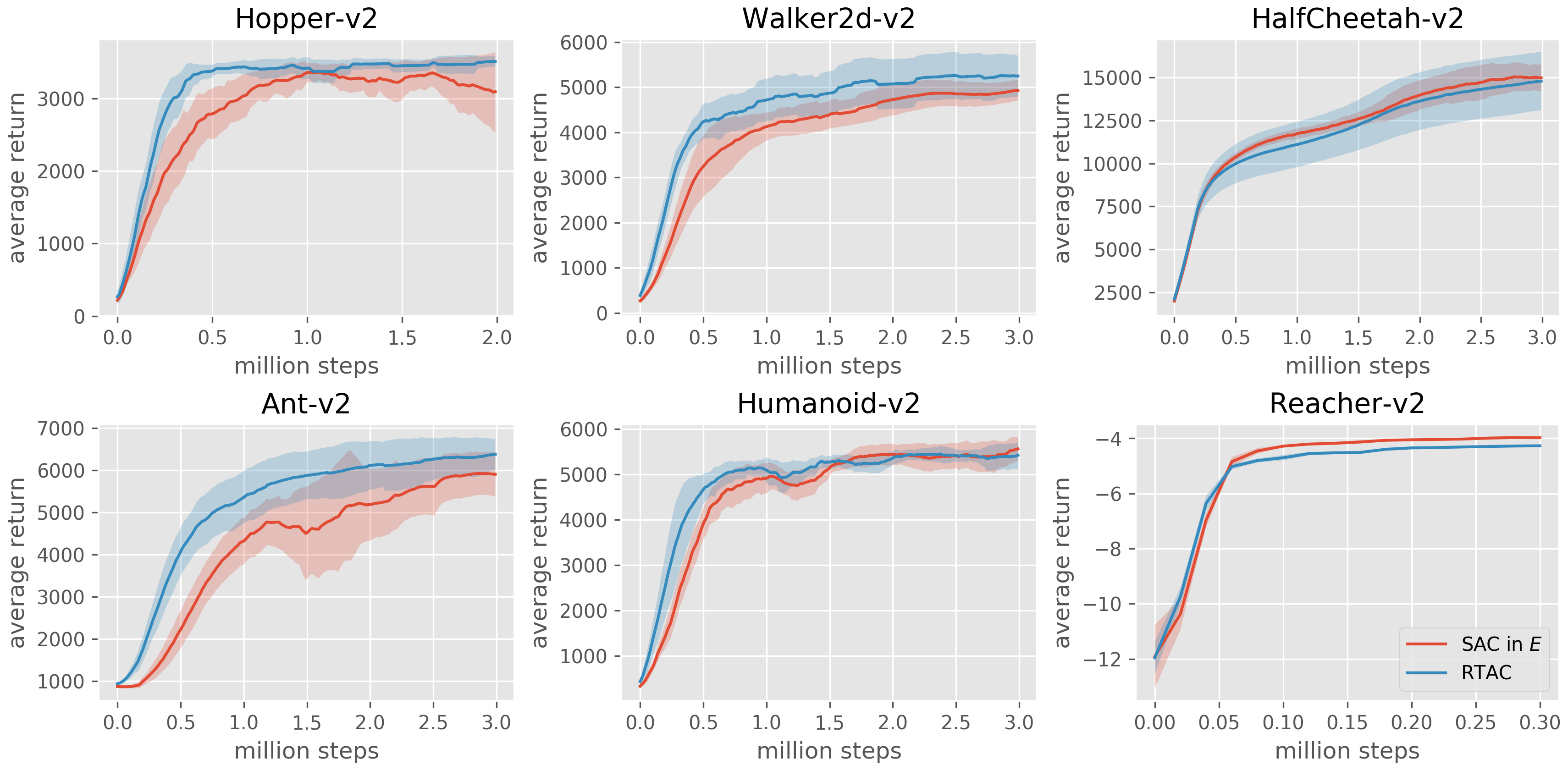}
  \caption{Comparison between RTAC (real-time) and SAC in $E$ (turn-based). Mean and $95\%$ confidence interval are computed over eight training runs per environment.}
  \label{RtacSacTb}
\end{figure}

\begin{figure}[H] 
  \centering
  \includegraphics[trim=0.7cm 0 0 0, width=1.02\linewidth]{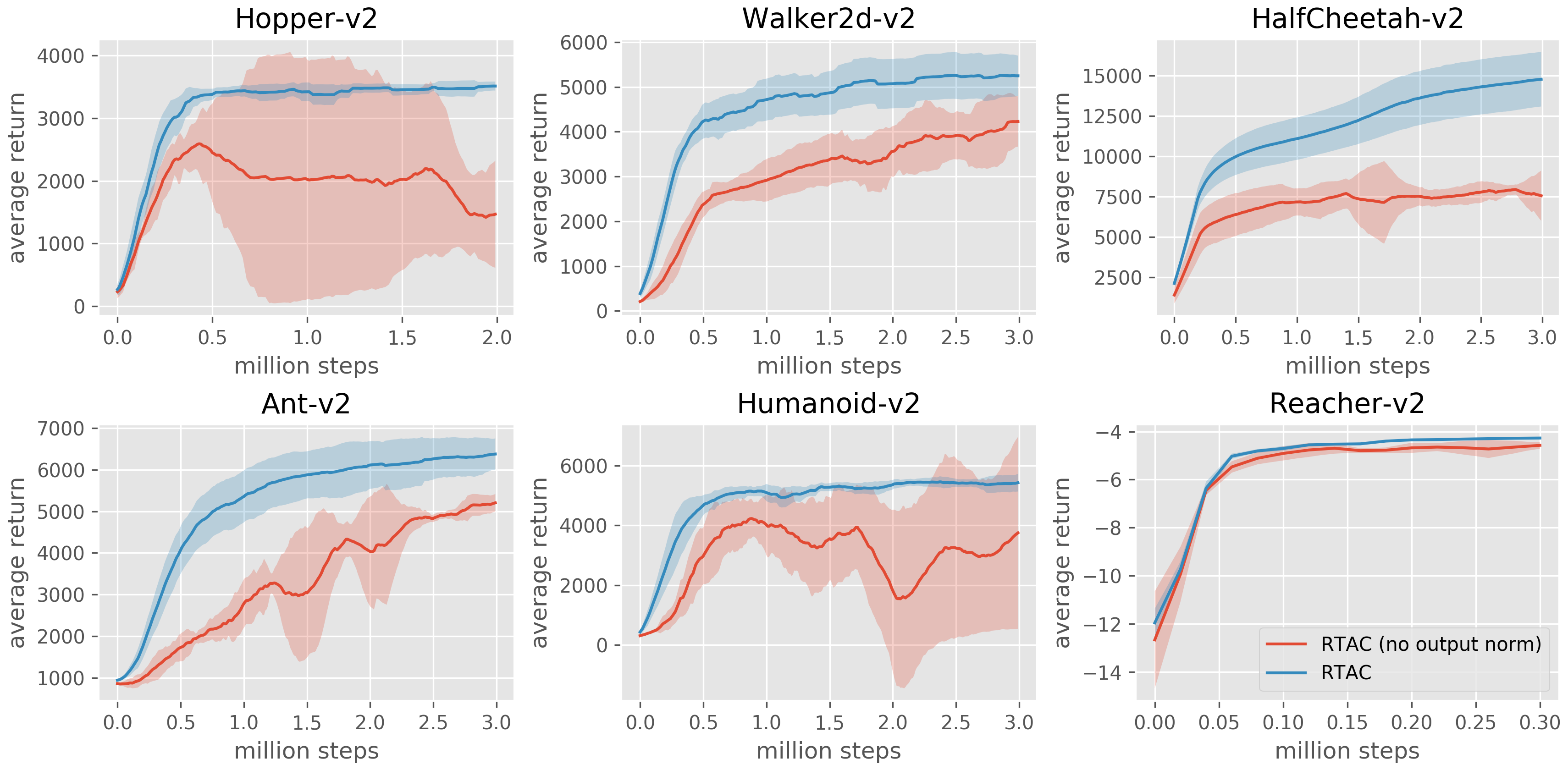}
  \caption{RTAC with and without output normalization. Mean and $95\%$ confidence interval are computed over eight and four training runs per environment, respectively.}
  \label{RtacNoNorm}
\end{figure}

\section{Hyperparameters}

\begin{table}[H]
  \caption{Hyperparameters}
  \label{Hyperparameters}
  \centering
  \begin{tabular}{lll}
    % \toprule
    % \multicolumn{2}{c}{Part}                   \\
    % \cmidrule(r){1-2}
    Name     & RTAC     & SAC \\
    \midrule
    optimizer & Adam & Adam \citep{kingma2014adam} \\
    learning rate & $0.0003$  & $0.0003$     \\
    discount ($\gamma$)     & $0.99$ & $0.99$      \\
    hidden layers     & $2$       & $2$  \\
    units per layer     & $256$       & $256$  \\
    samples per minibatch     & $256$       & $256$  \\
    target smoothing coefficient ($\tau$)     & $0.005$       & $0.005$  \\
    gradient steps / environment steps     & $1$       & $1$  \\
    reward scale & $5$ & $5$ \\
    entropy scale ($\alpha$) & $1$ & $1$ \\
    actor-critic loss factor ($\beta$)     & $0.2$       & -  \\
    Pop-Art alpha & $0.0003$ & - \\
    start training after & $10000$ & $10000$ steps \\
    \bottomrule
  \end{tabular}
\end{table}

\section{Proofs}
\label{proofs}

\TheoremRttb*
\begin{proof} For any environment $E=(S, A, \mu, p, r)$, we want to show that the two above MRPs are the same. Per Def.~\ref{def:TBMRP}~and~\ref{def:RTMDP} for $T\!B\!M\!R\!P(R\!T\!M\!D\!P(E), \augm \pi)$ we have
\begin{align*}
    &\text{(1) state space} && S \times A, \\
    &\text{(2) initial distribution} && \mu(s) \delta(a-c), \\
    &\text{(3) transition kernel} && \int_A p(s_{t+1} | s_t, a_t) \delta(a_{t+1} - \augm a) \ \augm\pi(\augm a | \itup{s_t, a_t}) \ d \augm a, \\
    &\text{(4) state-reward function}  && \int_A r(s, a) \ \augm\pi(\augm a | \itup{s_t, a_t}) \ d \augm a.
\end{align*}
The transition kernel, using the definition of the Dirac delta function $\delta$, can be simplified to
\begin{equation}
p(s_{t+1} | s_t, a_t) \int_A \delta(a_{t+1} - \augm a) \ \augm\pi(\augm a | \itup{s_t, a_t}) \ d \augm a = p(s_{t+1} | s_t, a_t) \ \augm \pi(a_{t+1} | \itup{s_t, a_t}).
\end{equation}
The state-reward function can be simplified to
\begin{equation}
r(s_t, a_t) \int_A \pi(\augm a | \x) \ d \augm a = r(s_t, a_t).
\end{equation}
It should now be easy to see how the elements above match $R\!T\!M\!R\!P(E, \augm \pi)$, Def.~\ref{def:RTMRP}.
\end{proof}

% \vspace{\baselineskip}
\pagebreak

\TheoremTbrt*
\begin{proof}
Given MDP $E = (S, A, \mu, p, r)$, we have {$\Psi = (Z, \nu, \sigma, \bar \rho) =  R\!T\!M\!R\!P(T\!B\!M\!D\!P(E), \augm \pi)$} with
\begin{align}
    &\text{(1) state space} && Z = S \times \{0, 1\} \times A, \\
    &\text{(2) initial distribution} && \nu(\itup{s, b, a}) = \mu(s) \ \delta(b) \ \delta(a-c), \\
    &\text{(3) transition kernel} && \sigma(\itup{s_{t+1}, b_{t+1}, a_{t+1}} | \itup{s_t, b_t, a_t}) \\
    & && \!\!\!\!\!\!\!\! = \begin{cases}
            \delta(s_{t+1}-s_t) \ \delta(b_{t+1}-1) \ \pi(a_{t+1}| s_t) & \text{if} \ b_t = 0 \\
            p(s_{t+1} | s_t, a_t) \ \ \ \delta(b_{t+1}) \quad \pi(a_{t+1} | s_t) & \text{if} \ b_t = 1
        \end{cases}, \\
    &\text{(4) state-reward function} && \bar\rho(\itup{s, b, a}) = r(s, a) \ b.
\end{align}
We can construct $\augm \Omega = (Z, \nu, \augm\kappa, \augm{\bar r})$, a sub-MRP with interval $n=2$. Since we always skip the step in which $b=1$, we only have to define the transition kernel for $b_t=0$, i.e.
\begin{align}
    \augm\kappa(z_{t+1} | z_t)
    & = \sigma^2 (\itup{s_{t+1}, b_{t+1}, a_{t+1}} | \itup{s_t, b_t, a_t}) \\
  % = & \int_Z \sigma(z_{t+1}|z') \sigma(z'|z_t) dz' \\
    & = \int_{S \times A} \sigma(\itup{s_{t+1}, b_{t+1}, a_{t+1}} | \itup{s', 1, a'}) \ \sigma(\itup{s', 1, a'} | \itup{s_t, 0, a_t}) \ d({s', a'}) \\
    & = \int_{S \times A} p(s_{t+1} | s', a') \ \delta(b_{t+1}) \ \pi(a_{t+1} | s') \ \delta(s'-s_t) \ \pi(a'| s_t) \ d(s', a') \\
    & = \int_A p(s_{t+1} | s_t, a') \ \delta(b_{t+1}) \ \pi(a'| s_t) \ d a'.
\end{align}
For the state-reward function we have (again only considering $b=0$)
\begin{align}
    \augm{\bar r}(\itup{s, b, a})
    & = v_\Psi^2(\itup{s, b, a}) \\
    & = \underbrace{\bar\rho(\itup{s, 0, a})}_{=0} \ + \int_{S \times A} \bar\rho(\itup{s', 1, a'}) \ \sigma(\itup{s', 1, a'}|\itup{s, 0, a}) \ d(s', a') \\
    & = \int_{S \times A}  r(s', a') \ \delta(s'-s) \ \pi(a'| s) \ d(s', a') \\
    & = \int_A r(s, a') \ \pi(a'|s) \ d a'.
\end{align}
The sub-MRP $\augm \Omega$ is already very similar to $T\!B\!M\!R\!P(E, \pi)$ except for having a larger state-space. To get rid of the $b$ and $a$ state components, we reduce $\augm \Omega$ with a state transformation $f(s, b, a) = s$. The reduced MRP has
\begin{align}
    & \text{(1) state space} && \{f(z) : z \in Z\} = S, \\
    & \text{(2) initial distribution} && \int_\dom{f^{-1}(s)} \nu(z) dz
        = \int_\dom{\{s\} \times \{0, 1\} \times A} \mu(s) \delta(b) \delta(a-c) \ d(s, b, a)
        = \mu(s), & \\
    & \text{(3) transition kernel} && \int_\dom{f^{-1}(s_{t+1})} \augm\kappa(z' | z) \ d z'
        \ \text{for almost all} \ z \in f^{-1}(s_t) \\
    & && \!\!\!\!\!\!\!\! = \int_\dom{\{s_{t+1}\} \times \{0, 1\} \times A} \augm\kappa(z' | z) \ d z'
        \ \text{for almost all} \ z \in \{s_t\} \times \{0, 1\} \times A \\
    & && \!\!\!\!\!\!\!\! = \int_A p(s_{t+1} | s_t, a') \ \pi(a'| s_t) \ d a' \\
    & \text{(4) state-reward function} && \augm{\bar r}(z)
        \ \text{for almost all} \ z \in f^{-1}(s). \\
    & && \!\!\!\!\!\!\!\! = \int_A r(s, a') \ \pi(a'|s) \ d a',
\end{align}
which is exactly $T\!B\!M\!R\!P(E, \pi)$.
\end{proof}
\vspace{\baselineskip}

\LemmaActionValue*
\begin{proof}
After starting with the definition of the action-value function for an environment {$(\augm X, A, \augm \mu, \augm p, \augm r) = \text{\textit{RTMDP}}(E)$} with $E = (S, A, \mu, p, r)$, we separate the transition distribution $\augm p$ into its two constituents $p$ and $\delta$ and then, integrate over the Dirac delta.
{\medmuskip=0mu \thinmuskip=0mu \thickmuskip=0mu
\begin{align}
    & q_{\text{\textit{RTMDP}}(E)}^{\augm \pi}(\augm x_t, \augm a_t)
    \ = \ q_{\text{\textit{RTMDP}}(E)}^{\augm \pi}(\itup{s_t, a_t}, \augm a_t) \\
    & \ = \ \augm r(\itup{s_t, a_t}, \augm a_t) + \E_{{\small \itup{s_{t+1}, a_{t+1}}} \sim \augm p(\cdot | {\small \itup{s_t, a_t}}, \augm a_t)}[ \underbrace{\E_{\augm a_{t+1} \sim \augm \pi(\cdot | {\small \itup{s_{t+1}, a_{t+1}}})}[ q_{\text{\textit{RTMDP}}(E)}^{\augm \pi}(\itup{s_{t+1}, a_{t+1}}, \augm a_{t+1})]}] \\
    & \ = \ r(s_t, a_t)  \quad +  \quad \int_S p(s_{t+1} | s_t, a_t) \quad \int_A \delta(a_{t+1}-\augm a_t) \quad \quad \quad \quad \quad \ \dots \quad \quad d a_{t+1} \ d s_{t+1} \\
    & \ = \ r(s_t, a_t) +  \int_S \!\! p(s_{t+1} | s_t, a_t) \ \ \E_{\augm a_{t+1} \sim \augm \pi(\cdot | {\small \itup{s_{t+1}, \augm a_{t}}})}[ q_{\text{\textit{RTMDP}}(E)}^{\augm \pi}(\itup{s_{t+1}, \augm a_{t}}, \augm a_{t+1})] \ \ d s_{t+1}
\end{align}}
\end{proof}

\LemmaStateValue*
\begin{proof}
We follow the same procedure as for Lemma~\ref{lemma:statevalue}.
{\medmuskip=0mu \thinmuskip=0mu \thickmuskip=0mu
\begin{align}
    & v_{\text{\textit{RTMDP}}(E)}^{\augm \pi}(\augm x_t)
    \ = \ v_{\text{\textit{RTMDP}}(E)}^{\augm \pi}(\itup{s_t, a_t}) \\
    & \ = \ \E_{\augm a_{t} \sim \augm \pi(\cdot | {\small \itup{s_{t}, a_{t}}})}[ \augm r(\itup{s_t, a_t}, \augm a_t) + \E_{{\small \itup{s_{t+1}, a_{t+1}}} \sim \augm p(\cdot | {\small \itup{s_t, a_t}}, \augm a_t)}[ v_{\text{\textit{RTMDP}}(E)}^{\augm \pi}(\itup{s_{t+1}, a_{t+1}})]] \\
    & \ = \ r(s_t, a_t) + \E_{\augm a_{t} \sim \augm \pi(\cdot | {\small \itup{s_{t}, a_{t}}})}[ \int_S p(s_{t+1} | s_t, a_t) \int_A \delta(a_{t+1}-\augm a_t) \ v_{\text{\textit{RTMDP}}(E)}^{\augm \pi}(\itup{s_{t+1}, a_{t+1}}) \ d a_{t+1} \ d s_{t+1}] \\
    & \ = \ r(s_t, a_t) +  \int_S \!\! p(s_{t+1} | s_t, a_t) \ \ \E_{\augm a_{t} \sim \augm \pi(\cdot | {\small \itup{s_{t}, a_{t}}})}[ v_{\text{\textit{RTMDP}}(E)}^{\augm \pi}(\itup{s_{t+1}, \augm a_{t}})] \ \ d s_{t+1}
\end{align}}
\end{proof}

\PropRtac*
\label{AppPropRtac}
\begin{proof} As shown in \citet{haarnoja2018soft}, Equation \ref{ActorLoss} can be reparameterized to obtain the policy gradient, which, applied in a RTMDP, yields
\begin{equation}\label{SacGrad}
    \nabla_{\augm \pi} L^\text{SAC}_{R\!T\!M\!D\!P(E), \augm \pi}
    = \E_{\augm x_t, \epsilon}[ \nabla_{\augm \pi} (\log {\augm \pi}(\augm h_{\augm \pi}(\augm x_t, \epsilon), \augm x_t) - \tfrac 1 {\alpha} \nabla_{\augm \pi} q(\augm x_t, \augm h_{\augm \pi}(\augm x_t, \epsilon)) ]
\end{equation}
and reparameterizing Equation \ref{RtacActorLoss} yields
\begin{equation}
    \nabla_{\augm \pi} L^\text{RTAC}_{R\!T\!M\!D\!P(E), \augm \pi} 
    = \E_{\augm x_t, \epsilon}[ \nabla_{\augm \pi} (\log {\augm \pi}(\augm h_{\augm \pi}(\augm x_t, \epsilon), \augm x_t) - \tfrac 1 {\alpha}\gamma \nabla_{\augm \pi} \E_{s_{t+1} \sim p(\cdot | \x_t)}[\augm v(s_{t+1}, \augm h_{\augm \pi}(\augm x_t, \epsilon))]]
\end{equation}
where $\augm h_{\augm \pi}$ is a function mapping from state and noise to an action distributed according to $\augm \pi$. This leaves us to show that
\begin{equation}
\nabla_{\augm a_t} q(\x_t, \augm a_t)
= \underbrace{\nabla_{\augm a_t} \augm r(\x_t, {\color{lightgray} \augm a_t})}_{= 0} + \nabla_{\augm a_t} \gamma \E_{\x_{t+1} \sim \augm p(\cdot | \x_t, \augm a_t)} [ \augm v(\x_{t+1}) ]
= \gamma \nabla_{\augm a_t}\E_{s_{t+1} \sim p(\cdot | \x_t)} [ \augm v(s_{t+1}, \augm a_t) ]
\end{equation}
which follows from the definition of the soft action-value function and simplifying quantities defined in the RTMDP.
% \algn{
% \nabla_{\augm a_t} q(\x_t, \augm a_t) = & \underbrace{\nabla_{\augm a_t} \augm r(\x_t, {\color{lightgray} \augm a_t})}_{= 0} + \nabla_{\augm a_t} \gamma \E_{\x_{t+1} \sim \augm p(\x_t, \augm a_t)} [ \augm v(\x_{t+1}) ] \\
% = & \nabla_{\augm a_t} \gamma \E_{s_{t+1} \sim p(\x_t)} [ \augm v(s_{t+1}, \augm a_t) ] \\
% % = & \nabla_{\augm a_t} \E_{s_{t+1} \sim p(\x_t)}[\augm v(s_{t+1}, \augm a_t)] \\
% }
\end{proof}

\section{Definitions}

\begin{definition} \label{def:TBMDP}
A Turn-Based Markov Decision Process $(Z, A, \nu, q, \rho) = T\!B\!M\!D\!P(E)$ augments another Markov Decision Process $E = (S, A, \mu, p, r)$, such that
\begin{flalign*}
    &\text{(1) state space} && Z = S \times \{0, 1\}, \\
    &\text{(2) action space} && A, \\
    &\text{(3) initial state distribution} && \nu(\itup{s_0, b_0}) = \mu(s_0) \ \delta(b_0), \\
    &\text{(4) transition distribution} && q(\itup{s_{t+1}, b_{t+1}} | \itup{s_t, b_t}, a_t) = \begin{cases}
            \delta(s_{t+1}-s_t) \ \delta(b_{t+1}-1) & \text{if} \ b_t = 0 \\
            p(s_{t+1} | s_t, a_t) \ \delta(b_{t+1}) & \text{if} \ b_t = 1
        \end{cases} \\
    &\text{(5) reward function} && \rho(\itup{s, b}, a) = r(s, a) \ b.
\end{flalign*}
\end{definition}
\vspace{\baselineskip}

\begin{definition} \label{def:sub-MRP}
$\augm \Omega = (Z, \nu, \augm \kappa, \augm {\bar r})$ is a sub-MRP of $\Psi = (Z, \nu, \sigma, \bar \rho)$ if its states are sub-sampled with interval $n \in \mathbb N$ and rewards are summed over each interval, i.e. for almost all $z$
\begin{equation}
    \augm\kappa(z'|z) = \kappa^n(z'|z) \quad \text{and} \quad \augm{\bar r}(z) = v_{\Psi}^n(z).
\end{equation}
% where $\kappa^n$ is the $n$-step transition function and $v_{\Omega}^n$ is the n-step value function.
\end{definition}
\vspace{\baselineskip}

\begin{definition} \label{def:MRP-reduction}
A MRP $\Omega = (S, \mu, \kappa, \bar r)$ is a reduction of $\augm \Omega = (Z, \nu, \augm\kappa, \augm{\bar r})$ if there is a state transformation $f: \augm Z \to S$ that neither affects the evolution of states nor the rewards, i.e.
% \begin{equation}
%     \mu(s) = \int_{f^{-1}(s)} \hspace{-.7cm} \nu(z) dz
%     \quad \text{and} \quad
%     \kappa(s | f(z)) = \int_{f^{-1}(s)} \hspace{-.7cm} \augm\kappa(z' | z) \ d z' \\
%     \quad \text{and} \quad
%     \bar r(f(z)) = \augm{\bar r}(z).
%     \footnote{Here, $f^{-1}(s)$ is the set of inputs for which $f(z) = s$, i.e. $\{z \in Z : f(z)=s\}$.} \\
% \end{equation}
\begin{flalign}
    & \text{(1) state space} && S = \{f(z) : z \in Z\}, & \\
    & \text{(2) initial distribution} && \mu(s) = \int_{f^{-1}(s)} \hspace{-.7cm} \nu(z) dz, & \\
    & \text{(3) transition kernel} && \kappa(s_{t+1} | s) = \int_{f^{-1}(s_{t+1})} \hspace{-.7cm} \augm\kappa(z' | z) \ d z'
        \ \ \text{for almost all} \ z \in f^{-1}(s), & \\
    & \text{(4) state-reward function} && r(s) = \augm{\bar r}(z)
        \ \ \text{for almost all} \ z \in f^{-1}(s). & 
\end{flalign}
% where $q_\Omega^t(s)$ is the state probability at time $t$.
\end{definition}
\vspace{\baselineskip}

\begin{definition} \label{def:MRP-contains}
A MRP $\Psi$ contains another MRP $\Omega$ (we write $\Omega \propto \Psi$) if $\Psi$ works at a higher frequency and has a richer state than $\Psi$ but behaves otherwise identically. More precisely,
\begin{equation}
   \Omega \propto \Psi  \iff \text{$\Omega$ is a reduction (Def.~\ref{def:MRP-reduction}) of a sub-MRP (Def.~\ref{def:sub-MRP}) of $\Psi$.}
\end{equation}
\end{definition}
\vspace{\baselineskip}

\begin{definition}
The $n$-step transition function of a MRP $\Omega = (S, \mu, \kappa, \bar r)$ is
\begin{equation}
    \kappa^n(s_{t+n} | s_t) = \int_S \kappa(s_{t+n} |s_{t+n-1}) \kappa^{n-1}(s_{t+n-1} | s_t) \ d s_{t+n-1}. \quad \text{\big | with $\kappa^1 = \kappa$}
\end{equation}
\end{definition}
\vspace{\baselineskip}

\begin{definition}
The n-step value function $v_{\Omega}^n$ of a MRP $\Omega = (S, \mu, \kappa, \bar r)$ is
\begin{equation}
    v_{\Omega}^n(s_t) = \bar r(s_t) + \int_S \kappa(s_{t+1}|s_t) v_{\Omega}^{n-1}(s_{t+1}) \ d s_{t+1}.
    \quad \text{\big | with $v_{\Omega}^1 = \bar r$}
\end{equation}
\end{definition}

\end{document}